\documentclass{article}

\usepackage{microtype}
\usepackage{graphicx}
\usepackage{booktabs} 
\usepackage{multirow}

\usepackage{hyperref}

\usepackage[accepted]{icml2024}

\usepackage{amsmath}
\usepackage{amssymb}
\usepackage{mathtools}
\usepackage{amsthm}

\usepackage{hyperref}
\usepackage{url}
\usepackage{float}

\usepackage{amssymb}
\usepackage{amsthm}
\usepackage{caption}
\usepackage{subcaption}
\usepackage{graphicx}
\usepackage[symbol]{footmisc}
\usepackage{cancel}
\hypersetup{colorlinks,citecolor=blue,linkcolor=red,urlcolor=blue}
\usepackage{mathtools}
\usepackage{tabularx}
\usepackage{booktabs} 

\newcommand{\bx}{\boldsymbol{x}}

\newcommand{\bz}{\boldsymbol{z}}

\newcommand{\bw}{\boldsymbol{\mathrm{w}}}

\newcommand{\bI}{\boldsymbol{\mathrm{I}}}

\usepackage{color, colortbl}
\definecolor{LightCyan}{rgb}{0.88,1,1}
\definecolor{LightYellow}{rgb}{1,1,0.88}
\definecolor{DarkYellow}{rgb}{0.8352941176470589,0.7137254901960784,0.0392156862745098}
\definecolor{Yellow}{rgb}{1,1,0.0784313725490196}
\definecolor{mydarkgreen}{RGB}{69, 153, 44}

\newcommand{\sign}{\mathrm{sign}}

\definecolor{green}{RGB}{0, 128, 0}
\definecolor{red}{RGB}{245, 60, 60}
\newcommand{\colorcellauroc}[1]{%
  \ifdim #1 pt < 0.5pt
    \cellcolor{red!100}
  \else
    \ifdim #1 pt < 0.51pt
      \cellcolor{red!100}
    \else
      \ifdim #1 pt < 0.52pt
        \cellcolor{red!20}
      \else
        \ifdim #1 pt < 0.55pt
          \cellcolor{red!10}
        \else
          \ifdim #1 pt < 0.6pt
            \cellcolor{green!10}
          \else
            \ifdim #1 pt < 0.7pt
              \cellcolor{green!20}
            \else
              \ifdim #1 pt < 0.8pt
                \cellcolor{green!40}
              \else
                \ifdim #1 pt < 0.9pt
                  \cellcolor{green!100}
                \else
                  \cellcolor{green!100}
                \fi
              \fi
            \fi
          \fi
        \fi
      \fi
    \fi
  \fi
  #1
}

\newcommand{\colorcellece}[1]{%
  \ifdim #1 pt < 0.02pt
    \cellcolor{green!45}
  \else
    \ifdim #1 pt < 0.05pt
      \cellcolor{green!40}
    \else
      \ifdim #1 pt < 0.07pt
        \cellcolor{green!35}
      \else
        \ifdim #1 pt < 0.10pt
          \cellcolor{green!30}
        \else
          \ifdim #1 pt < 0.15pt
            \cellcolor{green!25}
          \else
            \ifdim #1 pt < 0.20pt
              \cellcolor{green!20}
            \else
              \ifdim #1 pt < 0.25pt
                \cellcolor{red!50}
              \else
                \ifdim #1 pt < 0.30pt
                  \cellcolor{red!100}
                \else
                  \cellcolor{red!100}
                \fi
              \fi
            \fi
          \fi
        \fi
      \fi
    \fi
  \fi
  #1
}

\newcommand{\colorcellecetwo}[1]{%
  \ifdim #1 pt < 0.035pt
    \cellcolor{green!45}
  \else
    \ifdim #1 pt < 0.04pt
      \cellcolor{green!40}
    \else
      \ifdim #1 pt < 0.045pt
        \cellcolor{green!35}
      \else
        \ifdim #1 pt < 0.05pt
          \cellcolor{green!30}
        \else
          \ifdim #1 pt < 0.055pt
            \cellcolor{green!25}
          \else
            \ifdim #1 pt < 0.06pt
              \cellcolor{green!20}
            \else
              \ifdim #1 pt < 0.065pt
                \cellcolor{red!50}
              \else
                \ifdim #1 pt < 0.07pt
                  \cellcolor{red!100}
                \else
                  \cellcolor{red!100}
                \fi
              \fi
            \fi
          \fi
        \fi
      \fi
    \fi
  \fi
  #1
}

\newcommand{\colorcellecethree}[1]{%
  \ifdim #1 pt < 0.110pt
    \cellcolor{green!45}
  \else
    \ifdim #1 pt < 0.120pt
      \cellcolor{green!40}
    \else
      \ifdim #1 pt < 0.125pt
        \cellcolor{green!35}
      \else
        \ifdim #1 pt < 0.130pt
          \cellcolor{green!30}
        \else
          \ifdim #1 pt < 0.135pt
            \cellcolor{green!25}
          \else
            \ifdim #1 pt < 0.140pt
              \cellcolor{green!20}
            \else
              \ifdim #1 pt < 0.145pt
                \cellcolor{red!50}
              \else
                \ifdim #1 pt < 0.150pt
                  \cellcolor{red!100}
                \else
                  \cellcolor{red!100}
                \fi
              \fi
            \fi
          \fi
        \fi
      \fi
    \fi
  \fi
  #1
}

\newcommand{\colorcellecefour}[1]{%
  \ifdim #1 pt < 0.142pt
    \cellcolor{green!45}
  \else
    \ifdim #1 pt < 0.146pt
      \cellcolor{green!40}
    \else
      \ifdim #1 pt < 0.148pt
        \cellcolor{green!35}
      \else
        \ifdim #1 pt < 0.150pt
          \cellcolor{green!30}
        \else
          \ifdim #1 pt < 0.152pt
            \cellcolor{green!25}
          \else
            \ifdim #1 pt < 0.154pt
              \cellcolor{green!20}
            \else
              \ifdim #1 pt < 0.156pt
                \cellcolor{red!50}
              \else
                \ifdim #1 pt < 0.158pt
                  \cellcolor{red!100}
                \else
                  \cellcolor{red!100}
                \fi
              \fi
            \fi
          \fi
        \fi
      \fi
    \fi
  \fi
  #1
}

\newcommand{\colorcellecefive}[1]{%
  \ifdim #1 pt < 0.217pt
    \cellcolor{green!45}
  \else
    \ifdim #1 pt < 0.228pt
      \cellcolor{green!40}
    \else
      \ifdim #1 pt < 0.233pt
        \cellcolor{green!35}
      \else
        \ifdim #1 pt < 0.238pt
          \cellcolor{green!30}
        \else
          \ifdim #1 pt < 0.243pt
            \cellcolor{green!25}
          \else
            \ifdim #1 pt < 0.248pt
              \cellcolor{green!20}
            \else
              \ifdim #1 pt < 0.253pt
                \cellcolor{red!50}
              \else
                \ifdim #1 pt < 0.258pt
                  \cellcolor{red!100}
                \else
                  \cellcolor{red!100}
                \fi
              \fi
            \fi
          \fi
        \fi
      \fi
    \fi
  \fi
  #1
}

\newcommand{\colorcellecesix}[1]{%
  \ifdim #1 pt < 0.106pt
    \cellcolor{green!45}
  \else
    \ifdim #1 pt < 0.107pt
      \cellcolor{green!40}
    \else
      \ifdim #1 pt < 0.109pt
        \cellcolor{green!35}
      \else
        \ifdim #1 pt < 0.111pt
          \cellcolor{green!30}
        \else
          \ifdim #1 pt < 0.113pt
            \cellcolor{green!25}
          \else
            \ifdim #1 pt < 0.115pt
              \cellcolor{green!20}
            \else
              \ifdim #1 pt < 0.117pt
                \cellcolor{red!50}
              \else
                \ifdim #1 pt < 0.119pt
                  \cellcolor{red!100}
                \else
                  \cellcolor{red!100}
                \fi
              \fi
            \fi
          \fi
        \fi
      \fi
    \fi
  \fi
  #1
}

\newcommand{\colorcellbrierone}[1]{%
  \ifdim #1 pt < 0.4206pt
    \cellcolor{green!45}
  \else
    \ifdim #1 pt < 0.430pt
      \cellcolor{green!40}
    \else
      \ifdim #1 pt < 0.450pt
        \cellcolor{green!35}
      \else
        \ifdim #1 pt < 0.470pt
          \cellcolor{green!30}
        \else
          \ifdim #1 pt < 0.490pt
            \cellcolor{green!25}
          \else
            \ifdim #1 pt < 0.540pt
              \cellcolor{green!20}
            \else
              \ifdim #1 pt < 0.545pt
                \cellcolor{red!50}
              \else
                \ifdim #1 pt < 0.550pt
                  \cellcolor{red!100}
                \else
                  \cellcolor{red!100}
                \fi
              \fi
            \fi
          \fi
        \fi
      \fi
    \fi
  \fi
  #1
}

\newcommand{\colorcellbriertwo}[1]{%
  \ifdim #1 pt < 0.115pt
    \cellcolor{green!45}
  \else
    \ifdim #1 pt < 0.118pt
      \cellcolor{green!40}
    \else
      \ifdim #1 pt < 0.120pt
        \cellcolor{green!35}
      \else
        \ifdim #1 pt < 0.122pt
          \cellcolor{green!30}
        \else
          \ifdim #1 pt < 0.124pt
            \cellcolor{green!25}
          \else
            \ifdim #1 pt < 0.126pt
              \cellcolor{green!20}
            \else
              \ifdim #1 pt < 0.129pt
                \cellcolor{green!10}
              \else
                \ifdim #1 pt < 0.130pt
                  \cellcolor{red!100}
                \else
                  \cellcolor{red!100}
                \fi
              \fi
            \fi
          \fi
        \fi
      \fi
    \fi
  \fi
  #1
}

\newcommand{\colorcellbrierthree}[1]{%
  \ifdim #1 pt < 0.316pt
    \cellcolor{green!45}
  \else
    \ifdim #1 pt < 0.3195pt
      \cellcolor{green!40}
    \else
      \ifdim #1 pt < 0.3200pt
        \cellcolor{green!35}
      \else
        \ifdim #1 pt < 0.321pt
          \cellcolor{green!30}
        \else
          \ifdim #1 pt < 0.322pt
            \cellcolor{green!25}
          \else
            \ifdim #1 pt < 0.3230pt
              \cellcolor{green!20}
            \else
              \ifdim #1 pt < 0.324pt
                \cellcolor{red!50}
              \else
                \ifdim #1 pt < 0.325pt
                  \cellcolor{red!100}
                \else
                  \cellcolor{red!100}
                \fi
              \fi
            \fi
          \fi
        \fi
      \fi
    \fi
  \fi
  #1
}

\newcommand{\colorcellbrierfour}[1]{%
  \ifdim #1 pt < 0.477pt
    \cellcolor{green!45}
  \else
    \ifdim #1 pt < 0.485pt
      \cellcolor{green!40}
    \else
      \ifdim #1 pt < 0.490pt
        \cellcolor{green!35}
      \else
        \ifdim #1 pt < 0.495pt
          \cellcolor{green!30}
        \else
          \ifdim #1 pt < 0.500pt
            \cellcolor{green!25}
          \else
            \ifdim #1 pt < 0.505pt
              \cellcolor{green!20}
            \else
              \ifdim #1 pt < 0.510pt
                \cellcolor{red!50}
              \else
                \ifdim #1 pt < 0.514pt
                  \cellcolor{red!100}
                \else
                  \cellcolor{red!100}
                \fi
              \fi
            \fi
          \fi
        \fi
      \fi
    \fi
  \fi
  #1
}

\usepackage[capitalize,noabbrev]{cleveref}

\theoremstyle{plain}
\newtheorem{theorem}{Theorem}[section]
\newtheorem{proposition}[theorem]{Proposition}

\theoremstyle{definition}

\theoremstyle{remark}

\usepackage[textsize=tiny]{todonotes}

\icmltitlerunning{Just rotate it! Uncertainty estimation in closed-source models via multiple queries}

\begin{document}

\twocolumn[
\icmltitle{Just rotate it! Uncertainty estimation\\ in closed-source models via multiple queries
}



\icmlsetsymbol{equal}{*}

\begin{icmlauthorlist}
\icmlauthor{Konstantinos Pitas}{yyy}
\icmlauthor{Julyan Arbel}{yyy}
\end{icmlauthorlist}

\icmlaffiliation{yyy}{INRIA Grenoble Rhône-Alpes, France}

\icmlcorrespondingauthor{Julyan Arbel}{julyan.arbel@inria.fr}

\icmlkeywords{Machine Learning, ICML}

\vskip 0.3in
]



\printAffiliationsAndNotice{\icmlEqualContribution} 

\begin{abstract}
We propose a simple and effective method to estimate the uncertainty of closed-source deep neural network image classification models. Given a base image, our method creates multiple transformed versions and uses them to query the top-1 prediction of the closed-source model. We demonstrate significant improvements in calibration of uncertainty estimates compared to the naive baseline of assigning 100\% confidence to all predictions. While we initially explore Gaussian perturbations, our empirical findings indicate that natural transformations, such as rotations and elastic deformations, yield even better-calibrated predictions. 
Furthermore, through empirical results and a straightforward theoretical analysis, we elucidate the reasons behind the superior performance of natural transformations over Gaussian noise. Leveraging these insights, we propose a transfer learning approach that further improves our calibration results.
\end{abstract}

\section{Introduction}
In recent times, there has been a growing trend of monetizing and distributing deep learning models as closed-source applications or API-accessible websites \citep{team2023gemini,achiam2023gpt}. This shift has significantly amplified their impact in real-world scenarios, transforming them into practical tools for everyday use by diverse audiences. Simultaneously, the distribution of deep neural networks in a closed-source manner poses a challenge, as these models may lack the provision of uncertainty estimates crucial for certain applications and user needs \citep{abdar2021review,arbel2023primer}.

We narrow our focus to the domain of image classification, where given an image $\bx$ and a classifier $f: \mathbb{R}^d \rightarrow \mathcal{Y}$, our interest extends beyond identifying the most probable (top-1) class $f(\bx) = A$. We also seek a confidence level $p(A\vert \bx , f) \in [0,1]$, which quantifies the model's uncertainty regarding its prediction. In the absence of additional information and after a single query using the base image $\bx$, the naive baseline for closed-source models is to completely trust the model's prediction, and assign $p(A\vert \bx , f) = 1$.

In the broader context of uncertainty estimation for deep learning image classification models, existing methods often rely on accessing the post-softmax categorical distribution over classes \citep{guo2017calibration,lakshminarayanan2016simple,blundell2015weight,maddox2019simple,wenzel2020good}. However, such approaches are unsuitable for our setting, precisely due to the non-availability of the post-softmax categorical distribution. 

We thus propose a method to estimate the uncertainty of closed-source models by querying multiple times their top-1 prediction. Specifically we use transformed versions of $\bx$, generated through a randomized transformation function $\mathcal{T}: \mathbb{R}^d \rightarrow \mathbb{R}^d$ and estimate the fraction of times we observe class $A$ under the perturbation distribution $\mathcal{T}(\bx)\sim \tau$, denoted by $p_A$. At first, it might not be clear how to obtain a confidence estimate as a function of $p_A$. We therefore first analyze the simple case of adding Gaussian noise to our base image and link $p_A$ through a non-linear function to the categorical distribution of the model (we illustrate the approach in Figure \ref{fig:main}). The categorical distribution is a baseline measure of uncertainty for white-box deep neural networks \citep{bishop2006pattern,guo2017calibration}, and thus serves as a first justification of our approach.

\begin{figure*}[t!]
    \centering

    \begin{subfigure}{0.8\textwidth}
        \centering
        \includegraphics[width=\linewidth]{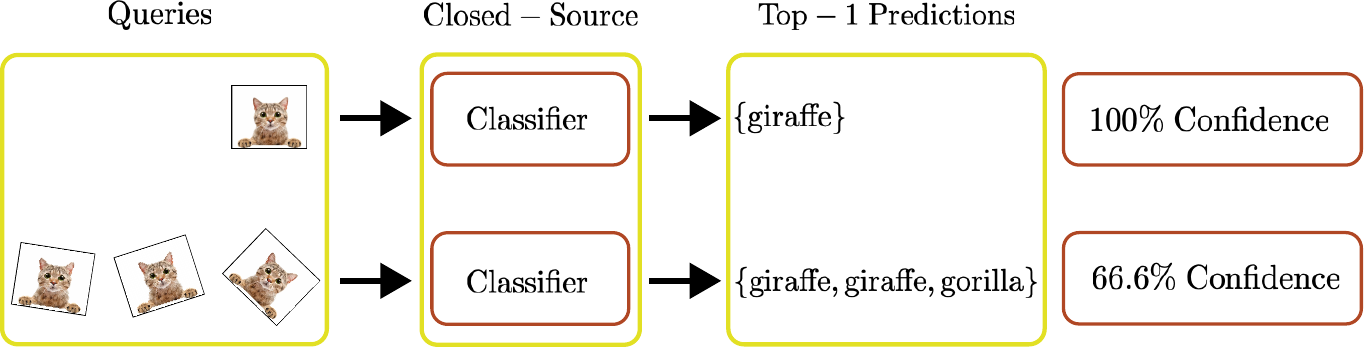}
        \label{fig:subfig1}
    \end{subfigure}

    \caption{\textbf{Rotational insights for multiple queries.} Top row: Querying a closed-source image classification model only once with a base image may yield an incorrect top-1 prediction. Despite the absence of additional information, the naive baseline is to assign $100\%$ confidence to this singular prediction. Bottom row: Querying the model multiple times with augmented versions of the base image produces the  $\{\mathrm{giraffe}\}$ class twice and the $\{\mathrm{gorilla}\}$ class once. This is roughly equivalent to $66.6\%$ confidence. This observation should serve as an alert to a potential error, even when the true label is unknown.}
    \label{fig:main}
\end{figure*}

We perform an in-depth analysis on CIFAR-10, CIFAR-100 \citep{krizhevsky2009learning}, Imagenet \citep{deng2009imagenet}, and widely-used architectures, demonstrating that test-time augmentation yields well-calibrated predictions according to standard calibration metrics, including the ECE \cite{naeini2015obtaining},  Brier score \cite{murphy1973new} and AUROC \citep{murphy2012machine}. Remarkably, in the case of Imagenet, even querying the model \emph{only twice} for each base image $\bx$ leads to enhanced calibration of predictions.

In evaluating the transformation function $\mathcal{T}$, we contrast Gaussian perturbations with natural transformations such as rotations, elastic deformations, and affine transforms. Our findings reveal that natural transformations exhibit superior performance compared to Gaussian perturbations. This holds particular significance, considering that natural transformations can be implemented without the need for coding--simply by rotating and translating the measurement device, e.g. a smartphone.

We then present a theoretical analysis that identifies two sources of error when linking $p_A$ to the categorical distribution. Firstly, considering a perturbation distribution $\mathcal{T}(\bx)\sim \tau$ at the input, it is essential for the noise distribution in the latent space $\eta$ \emph{just before the final classification layer} to be approximately consistent across different samples $\bx$. Secondly, we note that the optimal selection for the non-linear transformation between $p_A$ and the categorical distribution is contingent on $\eta$. An inadequate model of $\eta$ leads to suboptimal calibration in predictions.

Motivated by our previous theoretical analysis, we find that the suboptimal performance of the Gaussian perturbations compared to natural transformations is due to our inadequate model for $\eta$. We therefore consider a transfer learning scenario where we learn $\eta$ and the non-linear transformation between $p_A$ and $p(A\vert \bx , f)$ using CIFAR-10 and apply it on CIFAR-100. We get significant gains in calibration over our default non-linear transformation. In particular, by combining our learned non-linear function with Gaussian perturbations, we match or outperform the best natural transformations. 

Our analysis prescribes that practitioners first test their transformation function $\mathcal{T}$ of choice on open-source models and then use $\mathcal{T}$ and the learned non-linear function between $p_A$ and the categorical distribution on closed-source models.


In particular, our main contributions are as follows:
\begin{itemize}
\item We establish that multiple queries with augmented versions of an input $\bx$ can estimate the uncertainty of a closed-source model, and achieve up to $\sim30\%$ improvement in ECE, and $\sim40\%$ improvement in AUROC.
\item We show that natural transformations outperform Gaussian perturbations in uncertainty estimation by up to $\sim15\%$ in ECE and $\sim30\%$ in AUROC.
\item We analyze theoretically and empirically why natural transformations outperform Gaussian perturbations and find that the cause is an inadequate model of the latent noise $\eta$.
\item Given $\mathcal{T}$, we propose to learn $\eta$ using open-source models and then apply the $(\mathcal{T},\eta)$ combination on closed-source models. This leads to significant improvements when $\mathcal{T}$ is equivalent to Gaussian perturbations.
\end{itemize}

\section{Categorical distribution via multiple queries}

Let $f$ be a neural network classifier and $\bx$ be an input image. Assume that if we query using the base image $\bx$ the neural network predicts the top-1 class $f(\bx)=A$. Let $p(A\vert \bx , f)$ be the probability of class $A$ in the categorical distribution of $f$ given $\bx$ after applying the softmax function to the final network logits. 
We consider Gaussian perturbations to the input such that $\mathcal{T}(\bx)=\bx+\epsilon_{\tau}$ where $\epsilon_{\tau} \sim \mathcal{N}(0, \sigma_*^2\bI)$, and want to relate the $p(A\vert \bx , f)$, to the probability of sampling class $A$ under perturbations $\mathcal{T}(\bx)=\bx+\epsilon_{\tau}$, namely \begin{equation*}
p_A(\bx) = \mathbb{P}_{\epsilon_{\tau}}\left( f(\bx+\epsilon_{\tau}) = A\right).
\end{equation*}

We consider a decomposition of the neural network $f$ into $f = g \circ h$ where $h : \mathbb{R}^{d} \rightarrow \mathbb{R}^{d_{\eta}}$ is an encoder such that $\bz=h(\bx)$ is the latent representation of $\bx$ and $g : \mathbb{R}^{d_{\eta}} \rightarrow \mathcal{Y}$ is the final classification layer. We can show the following proposition for binary classification.

\begin{proposition}\label{main_theorem}
Let $f = g \circ h$ be a neural network used for binary classification, where $h : \mathbb{R}^{d} \rightarrow \mathbb{R}^{d_{\eta}}$ is an encoder and $g : \mathbb{R}^{d_{\eta}} \rightarrow \mathcal{Y}$ is the final classification layer, $\bx$ is an input image and $\epsilon_{\tau} \sim \mathcal{N}(0, \sigma_*^2\bI)$ is Gaussian noise added to the input. Let $\mathcal{J}_h$ be the Jacobian of $h$ at $\bx$,  and let $\bw$ be the separating hyperplane between class $A$ and class $B$.
Then the softmax probability outputted by $f$ given $\bx$, $p(A\vert \bx , f)$, can be related to the probability of sampling the original class $f(\bx)=A \in \mathcal{Y}$ under perturbations  $\epsilon_{\tau}$, $p_A(\bx)$,  by
\begin{equation}\label{confidence_impractical}
    p(A\vert \bx , f) \approx \frac{1}{1+e^{-\sigma_*\Vert \mathcal{J}_h(\bx)\bw \Vert \Phi^{-1}(p_A(\bx))}}.
\end{equation}
\end{proposition}

\begin{proof}[Proof Sketch]
We prove the simplified case where in the perturbations \emph{in latent space} are Gaussian $\bz+\epsilon_{\eta}\sim \mathcal{N}(\bz, \bI)$ and the separating hyperplane between classes $A$ and $B$ is $\bw^{\top}\bz+b$ with $\Vert \bw \Vert=1$. 

Since $p_A(\bx)$ is the fraction of times we sample $A$ under $\mathcal{T}$, to sample $A$ then $\bz+\epsilon_{\eta}$ \emph{should not be displaced from $\bz$ more than the distance to the decision boundary} (otherwise $\bz+\epsilon_{\eta}$ crosses the boundary and we predict $B$). The distance to the decision boundary is $\frac{\bw^{\top}(\bz+\epsilon_{\eta})+b}{\Vert \bw \Vert}=Y\sim \mathcal{N}(\bw^{\top}\bz+b,1)$. Therefore, taking into account that $Y$ is Gaussian, $p_A(\bx) = \mathbb{P}\left(Z>-\bw^{\top}\bz-b\right)=1-\Phi(-\bw^{\top}\bz-b) \iff \bw^{\top}\bz+b = \Phi^{-1}(p_A(\bx))$, where $\Phi^{-1}$ is the inverse cumulative of the Gaussian. 

By noting that in the binary classification case given a top-1 class prediction $A$, the softmax probability of this class is
\begin{align}
    &p(A\vert \bx , f) = \frac{e^{\bw_{A}^{\top}\bz+b_A}}{e^{\bw_{A}^{\top}\bz+b_A}+e^{\bw_{B}^{\top}\bz+b_B}}\nonumber\\
    &\,\,=\frac{1}{1+e^{-(\bw_{A}+b_A-\bw_{B}-b_B)^{\top}\bz}}=\frac{1}{1+e^{-(\bw^{\top}\bz+b)}},\label{eq:folk}
\end{align}
we get our result.
\end{proof}

\begin{table*}[!t]
    \caption{Results for the \textbf{CIFAR-10} and  \textbf{CIFAR-100} datasets and different architectures. Testing with augmentation always improves calibration, compared to the naive baseline of assigning 100\% confidence to all predictions. The best augmentation method is random rotations. $S$ denotes the number of samples from $\mathcal{T}(\bx)$. ``Var" denotes the variability measure $\sup_{x} \vert F_{n,Q_{2.5}}(x)-F_{n,Q_{97.5}}(x) \vert$. KS denotes the Kolmogorov--Smirnov test statistic $\sup_x \vert F^{-1}(x)-a\Phi^{-1}(x) \vert$ for the mean empirical cumulative $F_{n,\mathrm{mean}}(x)$ (the mean is with respect to the different $\bx$) and $\Phi(x/a)$ for a grid over $a$.}
    \label{table-res-cifar}
\centering
\begin{tabular}{cccccccccc}
\toprule
Dataset & Model & Augmentor & $S$ & Acc $\uparrow$ & ECE $\downarrow$ & AUROC $\uparrow$ & Brier $\downarrow$  & Var $\downarrow$ & KS $\downarrow$  \\
\hline
\multirow{5}{*}{CIFAR-10} & \multirow{5}{*}{MobileNet} & Naive & 1 & 0.940 & \colorcellecetwo{0.060}  & \colorcellauroc{0.500} & \colorcellbriertwo{0.119}  & - & -\\
 && Gaussian & 50 & 0.938 & \colorcellecetwo{0.053} & \colorcellauroc{0.603} & \textbf{\colorcellbriertwo{0.115}}  & 0.679 & 0.01\\
 && Affine & 50 & 0.933 & \colorcellecetwo{0.048} & \colorcellauroc{0.795} & \colorcellbriertwo{0.126}  & 0.735 & 0.116\\
 && Elastic & 50 & 0.913 & \colorcellecetwo{0.058} & \textbf{\colorcellauroc{0.898}} & \colorcellbriertwo{0.15}   & 0.709 & 0.147\\
 && Rotation & 50 & 0.932 & \textbf{\colorcellecetwo{0.035}} & \colorcellauroc{0.875} & \colorcellbriertwo{0.118}  & 0.743 & 0.202\\
\hline
\multirow{5}{*}{CIFAR-10} & \multirow{5}{*}{ResNet18} & Naive & 1 & 0.932 & \colorcellecetwo{0.068} & \colorcellauroc{0.500} & \colorcellbriertwo{0.137}  & - & -\\
 && Gaussian & 50 & 0.925 & \colorcellecetwo{0.044} & \colorcellauroc{0.698} & \colorcellbriertwo{0.132}  & 0.79 & 0.077\\
 && Affine & 50 & 0.924 & \colorcellecetwo{0.061} & \colorcellauroc{0.763} & \colorcellbriertwo{0.147}  & 0.604 & 0.05\\
 && Elastic & 50 & 0.909 & \colorcellecetwo{0.042} & \textbf{\colorcellauroc{0.903}} & \colorcellbriertwo{0.151}  & 0.59 & 0.075\\
 && Rotation & 50 & 0.924 & \textbf{\colorcellecetwo{0.027}} & \colorcellauroc{0.877} & \textbf{\colorcellbriertwo{0.128}}  & 0.69 & 0.072\\
\hline
\multirow{5}{*}{CIFAR-10} & \multirow{5}{*}{ResNet50} & Naive & 1 & 0.937 & \colorcellecetwo{0.063} & \colorcellauroc{0.500} & \colorcellbriertwo{0.126}  & - & -\\
 && Gaussian & 50 & 0.934 & \colorcellecetwo{0.046} & \colorcellauroc{0.651} & \colorcellbriertwo{0.120}   & 0.764 & 0.06\\
 && Affine & 50 & 0.928 & \colorcellecetwo{0.056} & \colorcellauroc{0.788} & \colorcellbriertwo{0.137}  & 0.667 & 0.11\\
 && Elastic & 50 & 0.919 & \colorcellecetwo{0.044} & \textbf{\colorcellauroc{0.901}} & \colorcellbriertwo{0.137}  & 0.675 & 0.148\\
 && Rotation & 50 & 0.932 & \textbf{\colorcellecetwo{0.028}} & \colorcellauroc{0.875} & \textbf{\colorcellbriertwo{0.117}}  & 0.712 & 0.083\\
 \hline
 \hline
\multirow{5}{*}{CIFAR-100} & \multirow{5}{*}{MobileNet} & Naive & 1 & 0.759 & \colorcellece{0.241} & \colorcellauroc{0.500} & \colorcellbrierone{0.482}  & - & -\\
 && Gaussian & 10 & 0.734 & \colorcellece{0.165} & \colorcellauroc{0.706} & \colorcellbrierone{0.450}  & 0.597 & 0.102\\
 && Affine & 10 & 0.721 & \colorcellece{0.075} & \colorcellauroc{0.726} & \colorcellbrierone{0.459}  & 0.273 & 0.050\\
 && Elastic & 10 & 0.687 & \colorcellece{0.050} & \textbf{\colorcellauroc{0.821}} & \colorcellbrierone{0.494}  & 0.302 & 0.046\\
 && Rotation & 10 & 0.731 & \textbf{\colorcellece{0.047}} & \colorcellauroc{0.811} & \textbf{\colorcellbrierone{0.427}}  & 0.392 & 0.031\\
\hline
\multirow{5}{*}{CIFAR-100} & \multirow{5}{*}{ResNet20} & Naive & 1 & 0.684 & \colorcellece{0.316} & \colorcellauroc{0.500} & \colorcellbrierone{0.632}  & - & -\\ 
 && Gaussian & 10 & 0.639 & \colorcellece{0.205} & \colorcellauroc{0.715} & \colorcellbrierone{0.590}  & 0.581 & 0.087\\
 && Affine & 10 & 0.657 & \colorcellece{0.038} & \colorcellauroc{0.680} & \colorcellbrierone{0.549}  & 0.264 & 0.050\\
 && Elastic & 10 & 0.625 & \textbf{\colorcellece{0.023}} & \textbf{\colorcellauroc{0.802}} & \colorcellbrierone{0.572}  & 0.295 & 0.050\\
 && Rotation & 10 & 0.659 & \colorcellece{0.047} & \colorcellauroc{0.778} & \textbf{\colorcellbrierone{0.533}}  & 0.389 & 0.031\\
\hline
\multirow{5}{*}{CIFAR-100} & \multirow{5}{*}{ResNet56} & Naive & 1 & 0.724 & \colorcellece{0.276} & \colorcellauroc{0.500} & \colorcellbrierone{0.553}  & - & -\\
 && Gaussian & 10 & 0.653 & \colorcellece{0.169} & \colorcellauroc{0.747} & \colorcellbrierone{0.552}  & 0.601 & 0.109\\
 && Affine & 10 & 0.677 & \colorcellece{0.062} & \colorcellauroc{0.710} & \colorcellbrierone{0.518}  & 0.259 & 0.050\\
 && Elastic & 10 & 0.663 & \textbf{\colorcellece{0.037}} & \textbf{\colorcellauroc{0.822}} & \colorcellbrierone{0.522}  & 0.236 & 0.044\\
 && Rotation & 10 & 0.693 & \colorcellece{0.044} & \colorcellauroc{0.801} & \textbf{\colorcellbrierone{0.481}}  & 0.392 & 0.030\\
\bottomrule
\end{tabular}
\end{table*}

The term $\sigma_*$ in \eqref{confidence_impractical} comes from perturbing the input with a Gaussian with variance $\sigma_*^2$. The term $ \mathcal{J}_h(\bx)$ results from modeling the Gaussian noise in input space as Gaussian noise in latent space, and the term $\bw$ results from considering $\bw$ with a general norm $\Vert\bw\Vert$. The full proof can be found in Appendix~\ref{appendix_additional_proofs}. The probability $p(A\vert \bx , f)$ that we estimate through $p_A(\bx)$ is a baseline measure of uncertainty for white-box deep neural networks \citep{bishop2006pattern,guo2017calibration}, and forms the basis for more complicated uncertainty estimate techniques such as temperature scaling \citep{guo2017calibration}, Bayesian neural networks \citep{blundell2015weight,arbel2023primer}, and deep ensembles \citep{lakshminarayanan2016simple}. As such Proposition \ref{main_theorem} serves as a first justification for using $p_A(\bx)$ as an estimate of uncertainty.

However, in practice, Equation \eqref{confidence_impractical} is not directly applicable as an estimate of uncertainty for closed-source classification models. In particular, the relationship approximating $p(A\vert \bx , f)$ depends on $\Vert \mathcal{J}_h(\bx)\bw \Vert$, which might be impossible to estimate using only queries of the top-1 class, contrary to $p_A(\bx)$. Therefore we derive the following practical model for uncertainty estimation.

\paragraph{Gaussian model.}
We will model the probability of the top-1 class as 
\begin{equation}\label{practical_method_1}
    p(A\vert \bx , f) = \frac{1}{1+e^{-a \Phi^{-1}(p_A(\bx))}}, \;\; \forall \bx,
\end{equation}

where $A=f(\bx)$ is the class predicted by $f$ on the base sample $\bx$, $p_A(\bx) = \mathbb{P}_{\mathcal{T}(\bx)}\left( f(\mathcal{T}(\bx)) = A\right)$, and $a$ is a hyperparameter common to all $\bx$ which we will estimate using a validation set. This model has the advantage of being simple to use in practice. Given a closed-source architecture and any transform in input space $\mathcal{T}(\bx)$ we simply need to learn the hyperparameter $a$ using a validation set to make useful uncertainty estimates on new data. 

In practice, we are interested in \textit{multiclass} classification problems. For such cases, it is reasonable to assume that the binary model holds approximately between the most probable class $A$ and the second most probable class $B$. To simplify our calculations and avoid estimating the second most probable class, let $C \in A^{c}$ be the set of all classes that are not class $A$. In multiclass classification problems, we compute $p_A(\bx)$, then $p(A\vert \bx , f)$ using Equation \eqref{practical_method_1}, and finally set $p(C\vert \bx , f)=(1-p(A\vert \bx , f))/\vert A^{c} \vert$. Note that the choice of $p(C\vert \bx , f)$ does not affect most calibration metrics, as they consider the confidence in the top-1 prediction.

\begin{table*}[t!]
    \caption{Results for the \textbf{Imagenet} dataset and different architectures. Testing with augmentation always improves calibration, compared to the naive baseline of assigning 100\% confidence to all predictions. The best augmentation method is random rotations. $S$ denotes the number of samples from $\mathcal{T}(\bx)$. ``Var" denotes the variability measure $\sup_{x} \vert F_{n,Q_{2.5}}(x)-F_{n,Q_{97.5}}(x) \vert$. KS denotes the Kolmogorov--Smirnov test statistic $\sup_x \vert F^{-1}(x)-a\Phi^{-1}(x) \vert$ for the mean empirical cumulative $F_{n,\mathrm{mean}}(x)$ (the mean is with respect to the different $\bx$) and $\Phi(x/a)$ for a grid over $a$.}
    \label{table-res-imagenet}
    \centering
\begin{tabular}{ccccccccc}
    \toprule
    Model & Augmentor & $S$ & Acc $\uparrow$ & ECE $\downarrow$ & AUROC $\uparrow$ & Brier $\downarrow$  & Var $\downarrow$ & KS $\downarrow$  \\
   \hline
\multirow{5}{*}{EfficientNet} & Naive & 1 & 0.842 & \colorcellecefour{0.158} & \colorcellauroc{0.500} & \textbf{\colorcellbrierthree{0.316}}  & - & -\\
 & Gaussian & 2 & 0.841 & \colorcellecefour{0.157} & \colorcellauroc{0.508} & \colorcellbrierthree{0.316}  & 0.641 & 0.042\\
 & Elastic & 2 & 0.835 & \colorcellecefour{0.148} & \colorcellauroc{0.564} & \colorcellbrierthree{0.316}  & 0.48 & 0.04\\
 & Affine & 2 & 0.822 & \textbf{\colorcellecefour{0.142}} & \textbf{\colorcellauroc{0.620}} & \colorcellbrierthree{0.325}  & 0.484 & 0.03\\
 & Rotation & 2 & 0.827 & \colorcellecefour{0.144} & \colorcellauroc{0.606} & \colorcellbrierthree{0.320}  & 0.525 & 0.029\\
\hline
\multirow{5}{*}{MobileNet} & Naive & 1 & 0.743 & \colorcellecefive{0.257} & \colorcellauroc{0.500} & \colorcellbrierfour{0.514}  & - & -\\
 & Gaussian & 2 & 0.74 & \colorcellecefive{0.248} & \colorcellauroc{0.529} & \colorcellbrierfour{0.508}  & 0.573 & 0.049\\
 & Elastic & 2 & 0.742 & \textbf{\colorcellecefive{0.217}}  & \colorcellauroc{0.601} & \textbf{\colorcellbrierfour{0.477}} & 0.401 & 0.052\\
 & Affine & 2 & 0.669 & \colorcellecefive{0.222} & \textbf{\colorcellauroc{0.697}} & \colorcellbrierfour{0.557}   & 0.384 & 0.046\\
 & Rotation & 2 & 0.707 & \colorcellecefive{0.219} & \colorcellauroc{0.659} & \colorcellbrierfour{0.514}  & 0.446 & 0.04\\
\hline
\multirow{5}{*}{ViTB16} & Naive & 1 & 0.853 & \colorcellecethree{0.147} & \colorcellauroc{0.500} & \colorcellbrierthree{0.295}  & - & -\\
 & Gaussian & 2 & 0.853 & \colorcellecethree{0.143} & \colorcellauroc{0.518} & \colorcellbrierthree{0.291}  & 0.603 & 0.034\\
 & Elastic & 2 & 0.846 & \colorcellecethree{0.137} & \colorcellauroc{0.572} & \colorcellbrierthree{0.292}  & 0.413 & 0.045\\
 & Affine & 2 & 0.818 & \colorcellecethree{0.135} & \textbf{\colorcellauroc{0.651}} & \colorcellbrierthree{0.325}  & 0.412 & 0.037\\
 & Rotation & 2 & 0.843 & \textbf{\colorcellecethree{0.128}} & \colorcellauroc{0.608} & \textbf{\colorcellbrierthree{0.290}}  & 0.462 & 0.034\\
\hline
\multirow{5}{*}{ViTL16} & Naive & 1 & 0.881 & \colorcellecesix{0.119} & \colorcellauroc{0.500} & \colorcellecefive{0.238}  & - & -\\
 & Gaussian & 2 & 0.881 & \colorcellecesix{0.117} & \colorcellauroc{0.508} & \colorcellecefive{0.236}  & 0.584 & 0.046\\
 & Elastic & 2 & 0.88 & \colorcellecesix{0.115} & \colorcellauroc{0.526} & \colorcellecefive{0.236} & 0.411 & 0.042\\
 & Affine & 2 & 0.872 & \textbf{\colorcellecesix{0.106}} & \textbf{\colorcellauroc{0.602}} & \colorcellecefive{0.239}  & 0.402 & 0.032\\
 & Rotation & 2 & 0.876 & \colorcellecesix{0.107} & \colorcellauroc{0.585} & \textbf{\colorcellecefive{0.235}}  & 0.475 & 0.030\\
\bottomrule
\end{tabular}
\end{table*}

\section{Gaussian model experiments}
We conducted experiments of increasing difficulty on the CIFAR-100, CIFAR-10, and Imagenet datasets. In the absence of additional information, the Gaussian model of Equation \eqref{practical_method_1} is used, not only with Gaussian, but also with \emph{natural} transformations such as elastic, rotation, and affine transformations (combined rotations, translations, and scaling). These transforms have the advantage of being applicable without coding by simply rotating and translating the measurement device. We provide details on the hyperparameter ranges for these transformations Appendix~\ref{appendix_setup_and_hyperparameter_ranges}. 

We test a variety of ResNet \citep{he2016deep}, MobileNet \cite{howard2017mobilenets}, Vision Transformer (ViT) \citep{dosovitskiy2020image}, and EfficientNet \citep{tan2019efficientnet} architectures. As a naive baseline, we query using the base image $\bx$ and assign $p(A\vert \bx , f)=1$ confidence to all predictions. To estimate $p_A(\bx)$ we use $S$ i.i.d. draws from $\mathcal{T}(\bx)$. For CIFAR-10 and CIFAR-100 we use a subset $m=1000$ of the test set to estimate $a$, and use the remaining $n=9000$ test samples to report test metrics. For Imagenet we report values on the publicly available validation set, which we subdivide into $m=5000$ and $n=45000$ respectively. We search $a\in\{0.001,0.005,0.01,0.05,0.1,0.5,1,10,100\}$, and perform the hyperparameter tuning over $a$ and the augmentation hyperparameters using grid search.

We measure calibration using the ECE, and the AUROC. The ECE and the AUROC however are not proper scoring rules \citep{nixon2019measuring} and thus clear improvements can only be validated for stable values of accuracy. If gains in ECE and AUROC are small or the accuracy changes significantly, then the Brier score \citep[which is a proper scoring rule, see ][]{murphy1973new} is more appropriate.

\begin{figure}[t!]
    \centering

    \begin{subfigure}{0.45\textwidth}
        \centering
        \includegraphics[width=\linewidth]{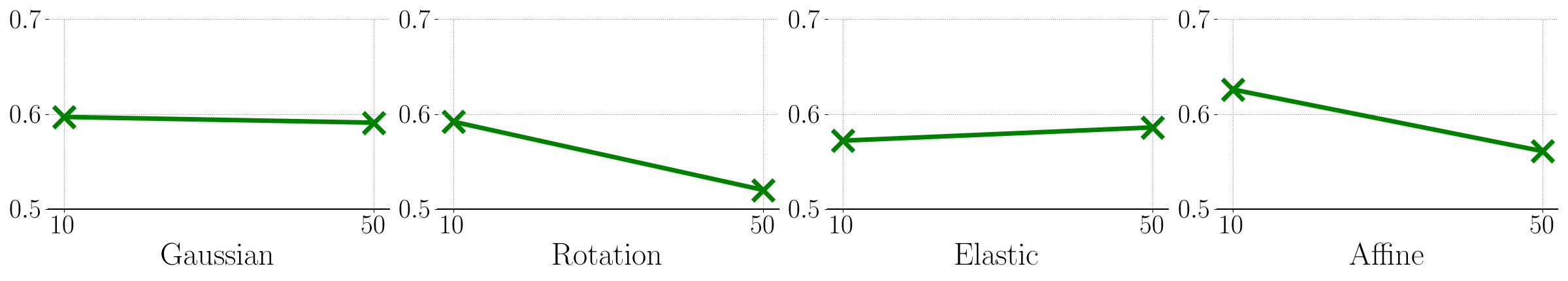}
        \caption{CIFAR-100, ResNet20}
        \label{fig:subfig12}
    \end{subfigure}

    \begin{subfigure}{0.45\textwidth}
        \centering
        \includegraphics[width=\linewidth]{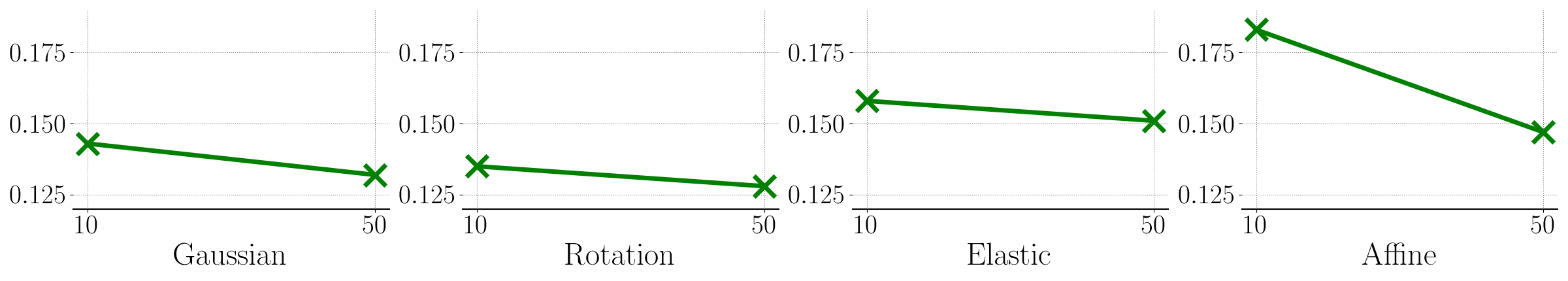}
        \caption{CIFAR-10, ResNet18}
        \label{fig:subfig22}
    \end{subfigure}

    \caption{\textbf{Brier score with varying number of samples.} Increasing the number of samples $S$ from $10$ to $50$ improves the Brier score consistently across datasets, architectures, and transformations.}
    \label{fig:main2}
\end{figure}

\begin{figure*}[t!]
    \centering
        \centering
        \includegraphics[width=.7\textwidth]{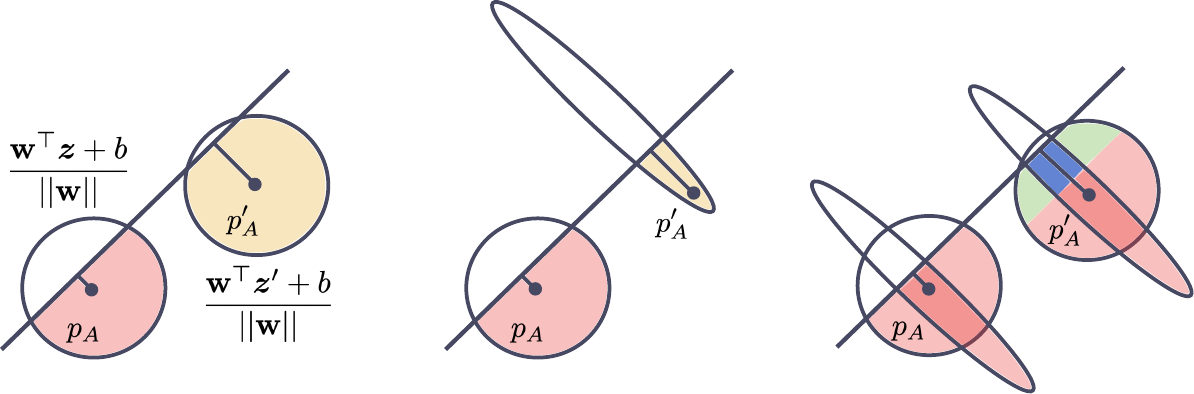}\\
        (a) Non-pathological distributions \hspace{.7cm} (b) Source of error \ref{error_source_1} \hspace{1.2cm}  (c) Source of error \ref{error_source_2}
        \label{fig:subfig1redun}
    \caption{\textbf{Noise distributions in latent space} $\epsilon_{\eta}$. (a) For non-pathological noise distributions, the fraction of samples from the top-1 class $p_A$ is commensurate to the margin of the decision boundary. (b) Error source \ref{error_source_1}: If the noise distributions $\epsilon_{\eta}$ and $\epsilon_{\eta}'$ in latent space for two signals $\bz$ and $\bz'$ are not identical, the inferred margins become incomparable. In this instance $\bz'$ has a larger margin than $\bz$, yet under $\epsilon_{\eta}'$, $p_A'<p_A$. (c) Error source \ref{error_source_2}: Different latent noise distributions correspond to different non-linear relationships between $p_A$ and the margin. Here, $\mathrm{margin}\;z' \approx 3 \times\mathrm{margin}\;z $. For the ellipsoid latent noise, $p_A' = p_A + \mathrm{``blue \; area"}\approx p_A+1/4 \times p_A = 1.25 \times p_A$. However, for the circular latent noise, $p_A' = p_A + \mathrm{``blue \; and \; green \; area"}\approx p_A+1/2 \times p_A = 1.5 \times p_A$. Firstly, in contrast to what one might initially guess, $p_A' \neq 3 \times p_A$ for both noise distributions. Secondly, the rate \( a \) at which \( p_A' \) changes, \( p_A' = a \times p_A \), also differs depending on the latent noise distribution. Therefore, an accurate model of $\epsilon_{\eta}$ is crucial for comparing the margins of $\bz$ and $\bz'$.}
    \label{fig:error_sources_complete}
\end{figure*}

The results are reported in Table \ref{table-res-cifar} and Table \ref{table-res-imagenet}. We see that in all cases test-time augmentation improves calibration while resulting in small changes in the clean accuracy. We distinguish between the easy CIFAR-100 case and the more difficult CIFAR-10 and Imagenet cases. For CIFAR-100 we get large improvements of up to $\sim30\%$ in ECE, $\sim32\%$ in AUROC and $0.10$ in Brier score, that result from the fact that the CIFAR-100 clean accuracy is relatively low. As such it is easier to detect potentially misclassified samples and improve calibration in terms of our calibration metrics. For the CIFAR-10 case, the challenge stems from having high clean accuracy. Intuitively, at the extreme, if the classifier is 100\% accurate on the test set, it is also 100\% calibrated, and thus test-time augmentation results in no gains. In terms of AUROC, we still get very large gains of up to $~40\%$. However, in terms of ECE, the gains are smaller $\sim3\%$ and since the accuracy drops by comparable values, it is more appropriate to compare the Brier score. For the Brier score, we get consistent improvements of $0.01$. For the Imagenet case, the difficulty stems from the fact that we choose $S=2$ which is the bare minimum for test-time augmentation. Note how we are then in the highly restricted uncertainty range $p(A\vert \bx , f)=\{0.25,0.5,0.75,1\}$. We gain up to $\sim3\%$ in ECE and $\sim20\%$ in AUROC. Notably, even for this minimal number of samples $S$ we still get improvements in Brier score of up to $\sim0.05$. 

In Figure \ref{fig:main2}, we plot how the Brier score improves when increasing the number of samples $S$ from $S=10$ to $S=50$. We see similar behavior for the other networks.
Notably, different transformations result in different gains in calibration. In particular, we note that natural transformations consistently outperform Gaussian perturbations by up to $\sim15\%$ in ECE and $\sim30\%$ in AUROC. Among natural transformations, rotations perform better than other transformations in terms of the Brier score.

\section{Two sources of error and the transfer learning model}

We now explore possible explanations for the differences in performance among the different transformations. We start by deriving a result  similar to Proposition \ref{main_theorem} for an \textit{unknown noise distribution in latent space} $\rho$, and then elucidate two important sources of error.

\begin{proposition}\label{theorem_general_noise}
Let $f = g \circ h$ be a neural network used for binary classification, where $h : \mathbb{R}^{d} \rightarrow \mathbb{R}^{d_{\eta}}$ is an encoder and $g : \mathbb{R}^{d_{\eta}} \rightarrow \mathcal{Y}$ is the final classification layer, $\bx$ is an input image and $\mathcal{T}(\bx)\sim \tau$ is some randomized transformation of the input. Let
\begin{equation}
p_A(\bx) = \mathbb{P}_{\mathcal{T}(\bx)}\left( f(\mathcal{T}(\bx)) = A\right),
\end{equation}
be the probability of sampling the original class $f(\bx)=A \in \mathcal{Y}$ under transformations $\mathcal{T}(\bx)\sim \tau$. Let $\epsilon_{\eta}$ be the corresponding perturbations to the latent representations and let $\bw$ be the separating hyperplane between class $A$ and class $B$. Assume that the inverse cumulative distribution $F^{-1}$ of $\bw^{\top}\epsilon_{\eta}\sim \rho$ exists and $\rho$ is zero-mean. Then
\begin{equation}
\begin{split}
    p(A\vert \bx , f) = \frac{1}{1+e^{F^{-1}(1-p_A(\bx))}},
\end{split}
\end{equation}
where $p(A\vert \bx , f)$ is the softmax probability outputted by $f$ given $\bx$.
\end{proposition}

Thus the Gaussian model of Equation \eqref{practical_method_1} is a subcase of Proposition \ref{theorem_general_noise} where we model $F^{-1}(x)$ as $a\Phi^{-1}(x)$. We now describe two important sources of error and methods to quantify them.

\begin{enumerate}
\item \label{error_source_1} $F^{-1}(x)$ depends on $\bx$. In the Gaussian model, we assumed that $F^{-1}(x)=a\Phi^{-1}(x), \; \forall \bx$ . If the noise in latent space is different across data samples, then the estimated confidence levels based on $p_A(\bx)$ are not directly comparable. To quantify this error we can first make an empirical estimate of $F$ for different $\bx$ through sampling and a Kolmogorov--Smirnov test. The empirical $F(x)$ are one-dimensional functions. We can thus use quantiles to estimate the spread of these functions. To summarize the variability of the different $F(x)$ into a single number we create the statistic $\mathrm{Var}=\max_x \vert Q_{2.5}(x)-Q_{97.5}(x) \vert$ where $Q_{2.5}(x)$ is the 2.5th percentile and $Q_{97.5}(x)$ and the 97.5th percentile. Intuitively this is the ``maximum spread" of the empirical cumulatives along the $y$-axis.

\item \label{error_source_2} Wrong model for latent noise, $F^{-1}(x)\neq a\Phi^{-1}(x)$. We've used the Gaussian model for different input transformations and thus it might be that $F^{-1}(x)$ is different based on the input transformations and provides uncalibrated estimates of $p(A\vert \bx, f)$ for some of them. We can use a Kolmogorov--Smirnov test to estimate $F^{-1}(x)$ empirically and compare it with $a\Phi^{-1}(x)$. We denote the Kolmogorov--Smirnov statistic $\sup_x \vert F^{-1}(x)-a\Phi^{-1}(x) \vert$ as KS.
\end{enumerate} 

\begin{figure*}[t!]
    \centering

    \begin{subfigure}{\textwidth}
        \centering
        \includegraphics[width=\linewidth]{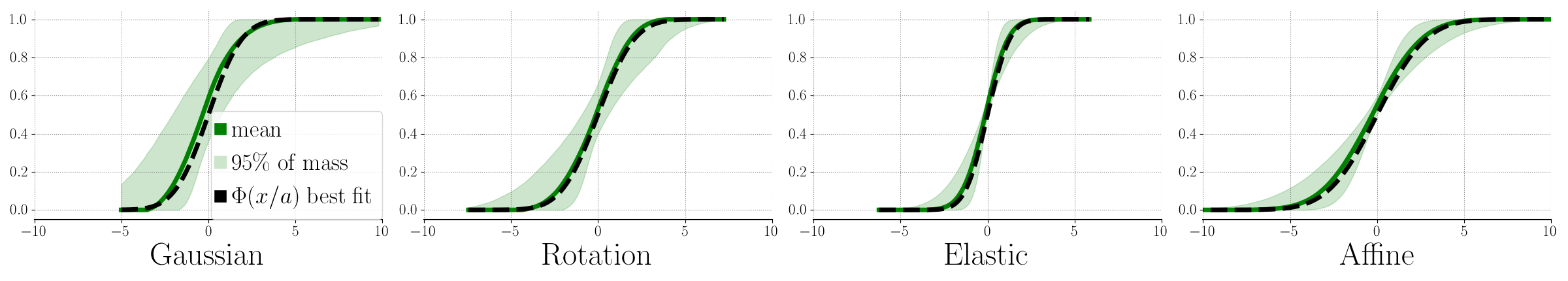}
        \caption{CIFAR-100, ResNet20}
        \label{fig:subfig11}
    \end{subfigure}

    \begin{subfigure}{\textwidth}
        \centering
        \includegraphics[width=\linewidth]{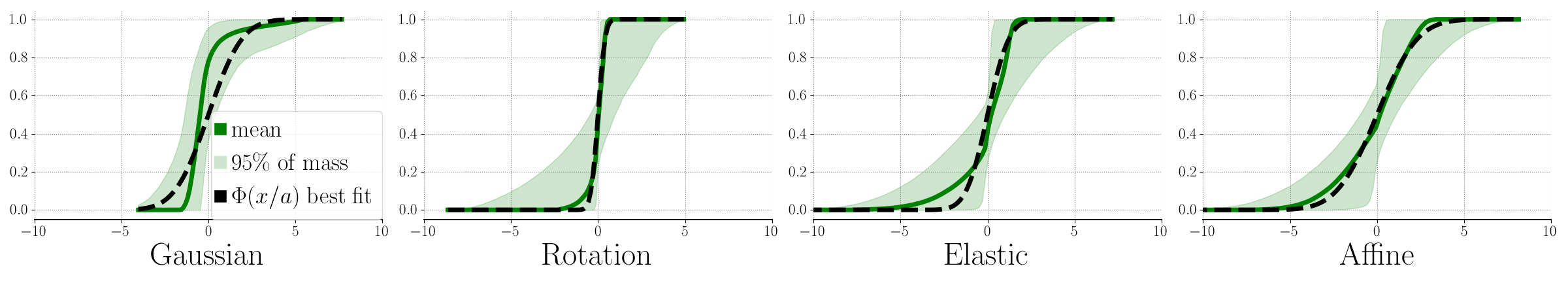}
        \caption{CIFAR-10, ResNet18}
        \label{fig:subfig21}
    \end{subfigure}

    \caption{\textbf{Empirical cumulatives} $F_n$ with quantiles for CIFAR-100, ResNet20 and CIFAR-10, ResNet18. For Gaussian noise on the input signals the mean empirical cumulative $F_n$ differs the most from the cumulative of the normal $\Phi(x/a)$. Furthermore, for Gaussian noise, the empirical cumulatives $F_n$ for the different samples $\bx$ exhibit the largest variability.}
    \label{fig:main1}
\end{figure*}

We plot a graphic illustration of these two sources of error in Figure \ref{fig:error_sources_complete} and also provide an alternative illustration in Appendix~\ref{appendix_graphical_illustration}. 
We explore our hypotheses about the two sources of error in explaining the sub-optimal performance of Gaussian perturbations. 

Error \ref{error_source_1}: In Figure \ref{fig:main1} we plot the cumulatives of $\rho$ for different $\bx$ and different transforms $\mathcal{T}(\bx)$. Specifically, for each case, given multiple cumulatives (one for each $\bx$) we plot the mean cumulative as well as a $95\%$ confidence interval \emph{along the $y$-axis}.  We see that the cumulatives for different $\bx$ have a larger variability in the Gaussian case compared to natural transformations (especially in Figure \ref{fig:subfig11}). 

Error \ref{error_source_2}: In Figure \ref{fig:main1} we also plot the best fit to the mean cumulative with the function $\Phi(x/a)$ (in terms of absolute discrepancy). We observe that for Gaussian perturbations, the cumulatives $F(x)$, especially for CIFAR-10, deviate significantly from the best fit of $\Phi(x/a)$. Specifically, the empirical estimate of $F(x)$ is asymmetric, with fatter tails towards positive numbers, while $\Phi(x/a)$ is symmetric. We also include our two statistics in the Tables as columns Var and KS, where we observe that they correlate with lower Brier score. 

In short, transformations that outperform in calibration either have more consistent noise distributions in latent space across samples $\bx$, or they result in noise distributions in latent space that more closely match our Gaussian model $\Phi(x/a)$. More precisely both the KS statistic and the Var statistic exhibit a statistically significant correlation with the ECE score across all datasets and architectures ($r=0.42,p=0.005$ and $r=0.46,p=0.002$ respectively), which we hope to exploit (we provide more data in Appendix~\ref{appendix_correlations_var_ks}).

Considering the above theoretical discussion we propose a transfer learning model to reduce Error \ref{error_source_2}. The main idea is to learn an empirical version of $F$, denoted $F_n$, which captures the asymmetries discovered in latent space for Gaussian perturbations.

\begin{table*}[t!]
    \caption{\textbf{Transfer learning.} Results for CIFAR-100 using transfer learning from CIFAR-10, with the best hyperparameters for the ECE. We first learn an empirical cumulative $F_n$ for Gaussian augmentations in the CIFAR-10 case, and the corresponding architectures. We then use the learned cumulative $F_n$ for Gaussian augmentations in the CIFAR-100 case, and its corresponding architectures. We chose the hyperparameter $a$ through the test-validation set so as to maximize accuracy and minimize the ECE. We see that we can get significant improvements in ECE with the learned $F_n$ compared to the case $F=\Phi$.}
    \label{table-res1}
    \centering
    \begin{tabular}{lclllllll}
    \toprule
    Model & Cumulative & $S$ & Acc $\uparrow$ & ECE $\downarrow$ & AUROC $\uparrow$ & Brier $\downarrow$ & KS $\downarrow$ \\
   \hline
\multirow{2}{*}{MobileNet} & $\Phi$ & 10 & 0.734 & \colorcellece{0.165} & \textbf{\colorcellauroc{0.706}}  & \colorcellbrierone{0.450} & 0.102\\
 & $F_n$ & 10 & 0.758 & \textbf{\colorcellece{0.010}} & \colorcellauroc{0.538} & \textbf{\colorcellbrierone{0.421}} & 0.069 \\
\hline
\multirow{2}{*}{ResNet20} & $\Phi$ & 10 & 0.639 & \colorcellece{0.205} & \textbf{\colorcellauroc{0.715}} & \colorcellbrierone{0.590} & 0.087  \\
 & $F_n$ & 10 & 0.672 & \textbf{\colorcellece{0.071}} & \colorcellauroc{0.637} & \textbf{\colorcellbrierone{0.536}} & 0.069 \\
\hline
\multirow{2}{*}{ResNet56} & $\Phi$ & 10 & 0.653 & \colorcellece{0.169} & \textbf{\colorcellauroc{0.747}} & \colorcellbrierone{0.552} & 0.109 \\
 & $F_n$ & 10 & 0.711 & \textbf{\colorcellece{0.053}}  & \colorcellauroc{0.647} & \textbf{\colorcellbrierone{0.479}} & 0.065 \\
\bottomrule
\end{tabular}
\end{table*}

\paragraph{Transfer Learning Model.}
Provided two datasets $\mathcal{D}_1$ and $\mathcal{D}_2$ (say CIFAR-10 and CIFAR-100), we estimate the empirical cumulative $F_n$ of the latent space noise distribution $\rho$ using a Kolmogorov--Smirnov test on $\mathcal{D}_1$, and then use it in a transfer learning setup on $\mathcal{D}_2$. Specifically, the probability of the top-1 class is modeled as 
\begin{equation}\label{practical_method_2}
    p(A\vert \bx , f) = \frac{1}{1+e^{a F_n^{-1}(1-p_A(\bx))}}, \;\; \forall \bx,
\end{equation}
where $A=f(\bx)$ is the class predicted by $f$ on the base sample $\bx$, $p_A(\bx) = \mathbb{P}_{\mathcal{T}(\bx)}\left( f(\mathcal{T}(\bx)) = A\right)$, and $a$ is a hyperparameter common to all $\bx$ which is estimated using a validation set.

\section{Transfer learning model experiments}

We use the transfer learning model to improve the calibration results of the Gaussian perturbations. Specifically, we learn an empirical cumulative $F_n$ from CIFAR-10 ($\mathcal{D}_1$) and use it on CIFAR-100 ($\mathcal{D}_2$) using Equation~\eqref{practical_method_2}. We again perform a grid-search over $a\in\{0.001,0.005,0.01,0.05,0.1,0.5,1,10,100\}$ and the hyperparameters of the Gaussian perturbations.

The results are reported in Table \ref{table-res1} where different calibration metrics are compared for $\Phi(x/a)$ and $F_n$. We see that the learned empirical distribution $F_n$ for the noise $\rho$ results in significant calibration improvements in terms of ECE, over $\Phi(x/a)$, validating our theoretical analysis. In particular, by combining our learned non-linear function with Gaussian perturbations, we match or outperform the best natural transformations in terms of ECE and Brier score. 

\section{Related work}

Our analysis is closely related to test-time augmentation which has previously been investigated for aleatoric uncertainty estimation in image segmentation by \citet{wang2019aleatoric}. However, it is noteworthy that this prior work does not specifically delve into the classification setting, nor does it explicitly address the closed-source model scenario. Furthermore, a detailed theoretical analysis, a focal point of our study, is not provided by these authors.

Other relevant works, including those by \citet{ashukha2019pitfalls,shanmugam2021better,kim2020learning}, have explored test-time augmentation to enhance uncertainty estimates. Notably, their focus lies in cases where the softmax categorical distribution can be readily accessed—a less complex scenario compared to the closed-source setting, which we rigorously analyze in this work. 

In the closed-source natural language generation setting, \citet{llmscan2023} and \citet{lin2023generating} propose to generate multiple answers for a given prompt and estimate the answer consistency as a measure of uncertainty.  \citet{lin2023generating} use a fixed prompt to generate multiple answers, while in a similar spirit to our work, \citet{llmscan2023} rephrase prompts  in some experiments to create noisy inputs. 

A comprehensive body of literature exists on the topic of estimating uncertainty in deep neural network models, when access to the softmax categorical distribution is available \citep{guo2017calibration, lakshminarayanan2016simple, blundell2015weight, maddox2019simple, wenzel2020good}. The most straightforward method involves utilizing the categorical distribution itself as an uncertainty estimate \citep{guo2017calibration}. Noteworthy enhancements can be achieved by applying tempering to the logits just before the application of the softmax function \citep{guo2017calibration}.

Various advanced methods have been proposed to refine uncertainty estimates, including techniques like deep ensembles \citep{lakshminarayanan2016simple} and Bayesian approaches, such as the Laplace approximation \citep{ritter2018scalable}, Markov chain Monte Carlo sampling \citep{wenzel2020good}, and Variational Inference \citep{blundell2015weight}. For an extensive overview of these methods, refer to \citet{arbel2023primer}. Fundamentally, these approaches share a common principle of averaging the softmax categorical distribution over multiple minima of the loss, contributing to improved uncertainty estimation in deep neural networks.


\section{Discussion}
We conducted a thorough analysis of multiple queries as a method for obtaining calibrated predictions from closed-source deep-learning models. Notably, we found that natural transformations, such as rotations, provide a straightforward means of identifying false positives. Their appeal lies in their practicality, as they can be implemented without coding, merely requiring the rotation and translation of the measurement device.

An intriguing direction for future research involves delving deeper into the connection between input transformations and the noise distribution in latent space. Exploring this relationship has the potential to further enhance the calibration of predictions, presenting an exciting avenue for continued investigation.

\bibliography{icml2024}
\bibliographystyle{icml2024}

\onecolumn
\appendix

\section{Additional proofs}\label{appendix_additional_proofs}
\begin{proposition}
Let $f = g \circ h$ be a neural network used for binary classification, where $h : \mathbb{R}^{d} \rightarrow \mathbb{R}^{d_{\eta}}$ is an encoder and $g : \mathbb{R}^{d_{\eta}} \rightarrow \mathcal{Y}$ is the final classification layer, $\bx$ is an input image and $\epsilon_{\tau} \sim \mathcal{N}(0, \sigma_*^2\bI)$ is Gaussian noise added to the input. Let $\mathcal{J}_h$ be the Jacobian of $h$ at $\bx$,  and let $\bw$ be the separating hyperplane between class $A$ and class $B$.
Then the softmax probability outputted by $f$ given $\bx$, $p(A\vert \bx , f)$, can be related to the probability of sampling the original class $f(\bx)=A \in \mathcal{Y}$ under perturbations  $\epsilon_{\tau}$, $p_A(\bx)$,  by
\begin{equation}
    p(A\vert \bx , f) \approx \frac{1}{1+e^{-\sigma_*\Vert \mathcal{J}_h(\bx)\bw \Vert \Phi^{-1}(p_A(\bx))}}.
\end{equation}
\end{proposition}

\begin{proof}
We can approximate the output of the neural network $f$ evaluated at the transformation $\mathcal{T}(\bx)$ of $\bx$ as
\begin{align*}
f(\mathcal{T}(\bx)) &= (g \circ h) (\bx+\epsilon_{\tau}) \approx g(h(\bx)+\mathcal{J}_h^{\top}\cdot \epsilon_{\tau}) \\
&= g(h(\bx)+\epsilon_{\eta}),
\end{align*}
where $\epsilon_{\eta} \sim \mathcal{N}(0,\sigma_*^2\mathcal{J}_h^{\top}\mathcal{J}_h)$, and  $\mathcal{J}_h$ is the Jacobian of $h$ at $\bx$.
Thus the input to the final classification layer is $\bz = h(\bx)$ perturbed with Gaussian noise $\epsilon_{\eta} \sim \mathcal{N}(0,\sigma_*^2\mathcal{J}_h^{\top}\mathcal{J}_h)$. We now turn to the categorical distribution, resulting from applying the softmax on the final layer logits. In the binary classification case given a top-1 class prediction $A$, the softmax probability of this class is
\begin{align}
    &p(A\vert \bx , f) = \frac{e^{\bw_{A}^{\top}\bz+b_A}}{e^{\bw_{A}^{\top}\bz+b_A}+e^{\bw_{B}^{\top}\bz+b_B}}\nonumber\\
    &\,\,=\frac{1}{1+e^{-(\bw_{A}+b_A-\bw_{B}-b_B)^{\top}\bz}}=\frac{1}{1+e^{-(\bw^{\top}\bz+b)}}.\label{eq:folk_app}
\end{align}
The above simply corresponds to the folk knowledge that a softmax layer with two classes is equivalent to a single separating hyperplane that assigns classes based on the rule $\sign \left(\bw^{\top}\bz+b\right)$, specifically 
$$g(\bz) = \begin{cases}
    A & \text{if } \left(\bw^{\top}\bz+b\right) > 0, \\
    B & \text{otherwise.}
\end{cases}$$ 
After establishing that the softmax layer is equivalent to this single separating hyperplane, let us relate $p_A(\bx)$ to $\bw^{\top}\bz+b$ and the norm $\Vert\mathcal{J}_h\bw\Vert$ using random sampling. We have
\begin{align*}
&p_A(\bx) =\mathbb{P}_{\epsilon_{\tau}}\left( f(\bx+\epsilon_{\tau}) = f(\bx)\right)\\
&=\,\, \mathbb{P}_{\epsilon_{\eta}}\left( \sign \left(\bw^{\top}(\bz+\epsilon_{\eta})+b\right) = \sign \left(\bw^{\top}\bz+b\right)\right).
\end{align*}
Since $f(\bx)=A \iff \sign (\bw^{\top}\bz+b) > 0$, then
\begin{align*}
p_A(\bx) &= \mathbb{P}_{\epsilon_{\eta}}\left( \bw^{\top}(\bz+\epsilon_{\eta})+b > 0 \right)\\
&= \mathbb{P}_{\epsilon_{\tau}}\left( \bw^{\top}(\bz+\mathcal{J}_h^{\top}\cdot \epsilon_{\tau})+b > 0 \right)\\
&= \Phi\left(\frac{\bw^{\top}\bz+b}{\sigma_*\Vert \mathcal{J}_h\bw \Vert} \right).    
\end{align*}
Thus $\bw^{\top}\bz+b = \sigma_*\Vert \mathcal{J}_h\bw \Vert \Phi^{-1}(p_A)$, combined with Equation~\eqref{eq:folk_app} this concludes the proof.
\end{proof}

\begin{proposition}
Let $f = g \circ h$ be a neural network used for binary classification, where $h : \mathbb{R}^{d} \rightarrow \mathbb{R}^{d_{\eta}}$ is an encoder and $g : \mathbb{R}^{d_{\eta}} \rightarrow \mathcal{Y}$ is the final classification layer, $\bx$ is an input image and $\mathcal{T}(\bx)\sim \tau$ is some randomized transformation of the input. Let
\begin{equation}
p_A = \mathbb{P}_{\mathcal{T}(\bx)}\left( f(\mathcal{T}(\bx)) = A\right),
\end{equation}
be the probability of sampling the original class $f(\bx)=A \in \mathcal{Y}$ under transformations $\mathcal{T}(\bx)\sim \tau$. Let the perturbations to the latent representations $\epsilon_{\eta}$ and let $\bw$ be the separating hyperplane between class $A$ and class $B$. Assume that the inverse cumulative distribution $F^{-1}$ of $\bw^{\top}\epsilon_{\eta}\sim \rho$ exists and $\rho$ is zero-mean. Then
\begin{equation}
\begin{split}
    p(A\vert \bx , f) = \frac{1}{1+e^{F^{-1}(1-p_A)}},
\end{split}
\end{equation}
where $p(A\vert \bx , f)$ is the softmax probability outputted by $f$ given $\bx$.
\end{proposition}

\begin{proof}
\begin{equation}
\begin{split}
p_A &= \mathbb{P}_{\epsilon_{\eta}}\left( \bw^{\top}(\bz+\epsilon_{\eta})+b > 0 \right)\\
&= \mathbb{P}_{\epsilon_{\eta}}\left( \bw^{\top}\bz+\bw^{\top}\epsilon_{\eta}+b > 0 \right)\\
&= \mathbb{P}_{\epsilon_{\eta}}\left( Z>-\bw^{\top}\bz-b\right)\\
&= 1-\mathbb{P}_{\epsilon_{\eta}}\left( Z<-\bw^{\top}\bz-b\right)\\
&= 1-F\left(-\bw^{\top}\bz-b\right)\\
\end{split}
\end{equation}
Then $F(-\bw^{\top}\bz-b) = 1-p_A \iff \bw^{\top}\bz+b = -F^{-1}(1-p_A)$.
\end{proof}

\section{Statistical correlations between Var, KS and calibration metrics}\label{appendix_correlations_var_ks}

\begin{table*}[t!]
    \caption{\textbf{Pearson's correlation to measure the effect of the Var and KS statistics on ECE, AUROC and Brier score.} When considering all datasets and architectures jointly, both the Var and the KS statistic exhibit statistically significant correlation with the ECE. For, all other cases a more nuanced analysis is required. Specifically for the AUROC and the Brier score, the correlation coefficients $r$ are in the correct direction, however the corresponding $p$-values, while low, are not statistically significant. This hints that we might need more data to assess these cases, or that there is no statistical correlation.}
    \label{table-correlation}
    \centering
    \begin{tabular}{lclllll}
    \toprule
    Model & Statistic & $S$ & ECE & AUROC  & Brier  \\
   \midrule
\multirow{2}{*}{CIFAR-100} & Var & 10 & $r=0.89,p=0.00007$ & $r=-0.26,p=0.39$ & $r=0.19,p=0.53$  \\
 & KS & 10  & $r=0.94,p=0.000002$ & $r=-0.54,p=0.06$ & $r=0.60,p=0.036$ \\
   \midrule
\multirow{2}{*}{CIFAR-10} & Var & 10 & $r=-0.44,p=0.14$ & $r=-0.23,p=0.45$ & $r=-0.53,p=0.07$  \\
 & KS & 10  & $r=-0.14,p=0.66$ & $r=-0.75,p=0.004$ & $r=0.60,p=0.036$ \\
   \midrule
\multirow{2}{*}{Imagenet} & Var & 10 & $r=0.73,p=0.001$ & $r=-0.75,p=0.0007$  & $r=-0.40,p=0.12$  \\
 & KS & 10  & $r=0.55,p=0.02$ & $r=-0.60,p=0.01$ & $r=-0.39,p=0.12$ \\
\midrule
\midrule
\multirow{2}{*}{All datasets} & Var & 10 & $r=0.42,p=0.005$ & $r=-0.45,p=0.003$  & $r=-0.26,p=0.10$  \\
 & KS & 10  & $r=0.46,p=0.002$ & $r=-0.17,p=0.26$ & $r=0.20,p=0.20$ \\
\bottomrule
\end{tabular}
\end{table*}

We conduct experiments to estimate the correlation between the Var and KS statistics and the ECE, AUROC and Brier score metrics. Specifically we estimate Pearson's correlation coefficient along with the associated p-values. We do this first for each dataset individually and then for all datasets jointly. We list these in Table \ref{table-correlation}. When considering all datasets and architectures jointly, both the Var and the KS statistic exhibit statistically significant correlation with the ECE. For, all other cases a more nuanced analysis is required. Specifically for the AUROC and the Brier score, the correlation coefficients $r$ are in the correct direction, however the corresponding $p$-values, while low, are not statistically significant. This hints that we might need more data to assess these cases, or that there is no statistical correlation.

\section{Experimental setup and hyperparameter ranges}\label{appendix_setup_and_hyperparameter_ranges}
All experiments were run on our local cluster, on NVIDIA V100 GPUs. The total computational budget was 1000 GPU hours.

We investigated the following hyperparameter ranges in all experiments.
\begin{itemize}
\item for the hyperparameter $a$ we explored $a=\{0.001,0.005,0.01,0.05,0.1,0.5,1,10,100\}$.
\item For the case of Gaussian perturbations, we investigated $\sigma=\{0.01,0.05,0.1,0.12,0.14,0.16,0.18,0.2\}$.
\item For the case of Elastic deformations we explore $a_{\epsilon}=10,20,50,70$ and $\sigma_{\epsilon}=2,5,10$.
\item For the case of Rotations we explore $\mathrm{degrees}=\{10,20,30,40,50,60\}$.
\item For the case of Affine transformations we explore $\mathrm{degrees}=\{0,10,30\}$, $\mathrm{translation}=\{0,0.1,0.3\}$, $\mathrm{scale}=\{0,0.1,0.3,1\}$.
\end{itemize}

For hyperparameter optimization we performed a grid search over the parameters.

\section{Graphical illustration of the two test statistics (KS,Var)}\label{appendix_graphical_illustration}

\begin{enumerate}
\item Figure \ref{fig:statistics:subfig1}: $F^{-1}(x)$ depends on $\bx$. In the Gaussian model, we assumed that $F^{-1}(x)=\hat{F}^{-1}(x), \; \forall \bx$ . If the noise in latent space is different across data samples, then the estimated confidence levels based on $p_A(\bx)$ are not directly comparable. To quantify this error we can first make an empirical estimate of $F$ for different $\bx$ through sampling and a Kolmogorov-Smirnov test \cite{sprent2007applied}. The empirical $F(x)$ are one-dimensional functions. We can thus use quantiles to estimate the spread of these functions. To summarize the variability of the different $F(x)$ into a single number we create the statistic $\mathrm{Var}=\max_x \vert Q_{2.5}(x)-Q_{97.5}(x) \vert$ where $Q_{2.5}(x)$ is the 2.5th quantile and $Q_{97.5}(x)$ and the 97.5th quantile. Intuitively this is the ``maximum spread" of the empirical cumulatives along the $y$-axis.

\item Figure \ref{fig:statistics:subfig2}: Given a model cumulative $\hat{F}^{-1}(x)$ and the true cumulative $F^{-1}(x)$ we might have the wrong model for latent noise, $F^{-1}(x)\neq \hat{F}^{-1}(x)$. It might be that $F^{-1}(x)$ is different based on the input transformations and provides uncalibrated estimates of $p(A\vert \bx, f)$ for some of these transformation distributions. We can use a Kolmogorov-Smirnov test to estimate $F^{-1}(x)$ empirically and compare it with $\hat{F}^{-1}(x)$. We denote the Kolmogorov-Smirnov statistic $\sup_x \vert F^{-1}(x)-\hat{F}^{-1}(x) \vert$ as KS.
\end{enumerate} 

\begin{figure*}[t!]
    \centering
        \centering
        \includegraphics[width=.7\textwidth]{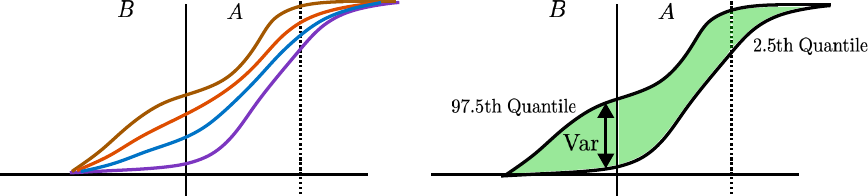}
        \caption{(a) Multiple cumulatives \hspace{3cm} (b) Var statistic}
        \label{fig:statistics:subfig1}

        \centering
        \includegraphics[width=.7\textwidth]{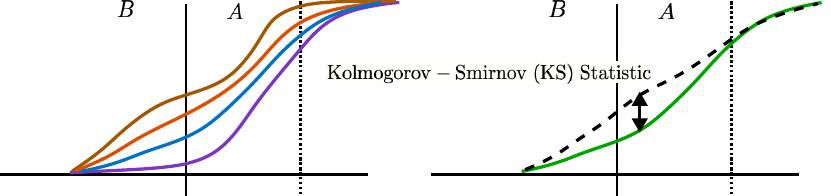}
        \caption{(a) Multiple cumulatives \hspace{3cm} (b) KS statistic}
        \label{fig:statistics:subfig2}
    \caption{\textbf{A graphical explanation of the KS and Var statistics.} Subfigure \ref{fig:statistics:subfig1}: The Var statistic quantifies the spread of 95\% of the mass of the different cumulatives across the $y$-axis. Subfigure \ref{fig:statistics:subfig2}: The KS statistic quanitfies the maximum absolute difference between our model cumulative and the empirical one.}
    \label{fig:statistics}
\end{figure*}

\end{document}